\documentclass[10pt,twocolumn,letterpaper]{article}

\usepackage{cvpr}
\usepackage{times}
\usepackage{epsfig}
\usepackage{graphicx}
\usepackage{amsmath}
\usepackage{amssymb}
\usepackage{subfigure}
\usepackage{url}
\usepackage{times}
\usepackage{epsfig}
\usepackage{graphicx}
\usepackage{algorithmic}
\usepackage{algorithm}
\usepackage{multirow}
\usepackage{authblk}

\newtheorem{theorem}{Theorem}[section]
\newtheorem{lemma}[theorem]{Lemma}

\newenvironment{proof}[1][Proof]{\begin{trivlist}
\item[\hskip \labelsep {\bfseries #1}]}{\end{trivlist}}

\newcommand{\qed}{\nobreak \ifvmode \relax \else
      \ifdim\lastskip<1.5em \hskip-\lastskip
      \hskip1.5em plus0em minus0.5em \fi \nobreak
      \vrule height0.75em width0.5em depth0.25em\fi}

\newcommand{\comment}[1]{}

\DeclareMathOperator*{\argmin}{argmin}

\usepackage[svgnames]{xcolor} 
\usepackage{graphicx}

\cvprfinalcopy 

\def\cvprPaperID{622} 

\ifcvprfinal\fi
\begin{document}
\onecolumn

\newpage

\newcommand*{\plogo}{\fbox{$\mathcal{PL}$}} 


\newcommand*{\titleAT}{\begingroup 
\newlength{\drop} 
\drop=0.05\textheight 

\includegraphics[scale=1.5]{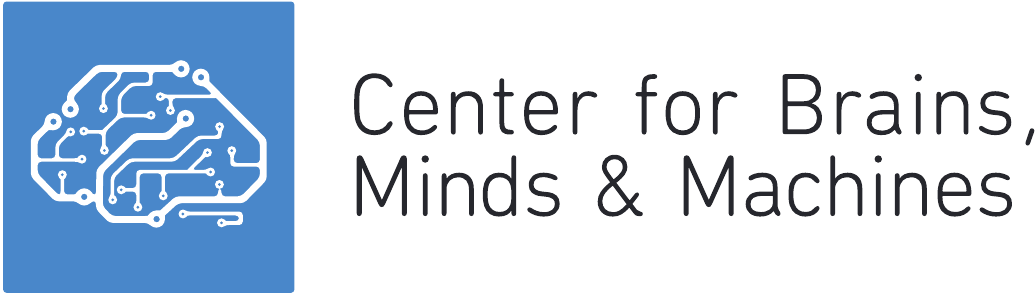}

\textcolor{CornflowerBlue}{\rule{\textwidth}{3 pt}}\par 
\vspace{2pt}\vspace{-\baselineskip} 
\rule{\textwidth}{0.4pt}\par 

\vspace{\drop} 
\textbf{\textsf{\large{CBMM Memo No. \memonumber}}}\quad \quad \quad\quad \quad \quad \quad\quad\quad \quad\quad\quad      \textbf{\large{\memodate}}

\vspace{\drop}
\begin{center}
\textbf{\textsf{\huge{\memotitle}}}\\
\vspace{0.4\drop}
\textbf{\Large{\textsf{by}}}\\
\vspace{0.4\drop}
\textbf{\textsf{\large{\memoauthors}}}
\end{center}
\vspace{\drop}
\textbf{\textsf{\large{\noindent Abstract}:}} {\memoabstract}

\textcolor{CornflowerBlue}{\rule{\textwidth}{3 pt}}\par 
\vspace{2pt}\vspace{-\baselineskip} 
\rule{\textwidth}{0.4pt}\par

\begin{minipage}{.15\linewidth}
\includegraphics[scale=0.1]{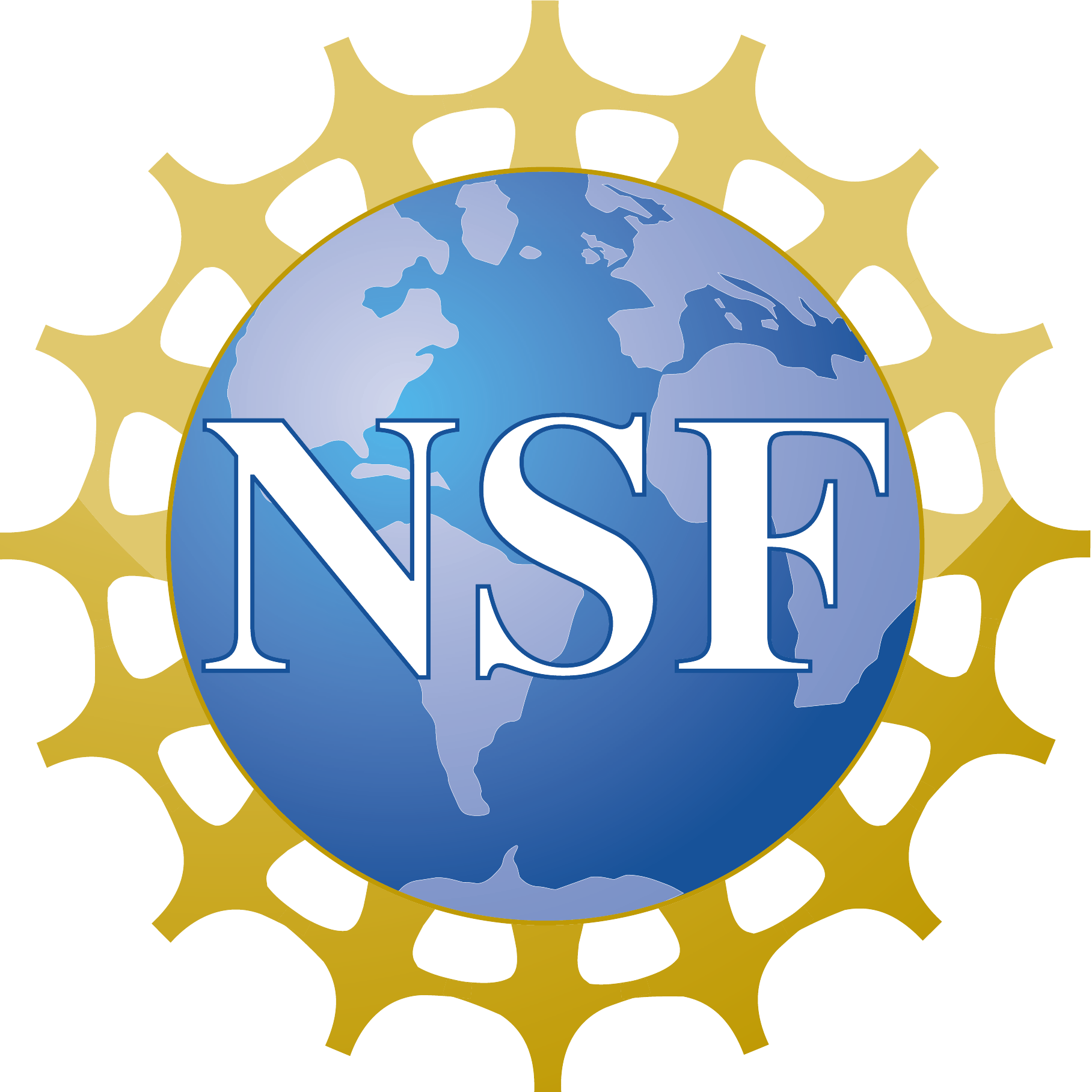}
\end{minipage}
\begin{minipage}{.84\linewidth}
\textbf{\textsf{\large{This work was supported by the Center for Brains, Minds and Machines (CBMM), funded by NSF STC award  CCF - 1231216.}}}
\end{minipage}
\endgroup}


\pagestyle{empty} 

\def\eg{\textsf{\it e.g.}}
\def\ie{\textsf{\it i.e.}}

\def\memonumber{ \textsf{013}} 
\def\memodate{\textsf{\today}} 
\def\memotitle{\textsf{Robust Estimation of 3D Human Poses from a Single Image}} 
\def\memoauthors{\textsf{Chunyu Wang$^{1}$,Yizhou Wang$^{1}$,Zhouchen Lin$^{1}$,Alan L. Yuille$^{2}$,Wen Gao$^{1}$}\\
$^{1}$Peking University, Beijing, China~~
$^{2}$University of California, Los Angeles\\
{\tt\small \{wangchunyu,yizhou.wang, zlin\}@pku.edu.cn~@yuille@stat.ucla.edu~wgao@pku.edu.cn}}

\def\memoabstract{\textsf{Human pose estimation is a key step to action recognition. We propose a method of estimating 3D human poses from a single image, which works in conjunction with an existing 2D pose/joint detector. 3D pose estimation is challenging because multiple 3D poses may correspond to the same 2D pose after projection due to the lack of depth information. Moreover, current 2D pose estimators are usually inaccurate which may cause errors in the 3D estimation. We address the challenges in three ways: (i) We represent a 3D pose as a linear combination of a sparse set of bases learned from 3D human skeletons. (ii) We enforce limb length constraints to eliminate anthropomorphically implausible skeletons. (iii) We estimate a 3D pose by minimizing the $L_1$-norm error between the projection of the 3D pose and the corresponding 2D detection. The $L_1$-norm loss term is robust to inaccurate 2D joint estimations. We use the alternating direction method (ADM) to solve the optimization problem efficiently. Our approach outperforms the state-of-the-arts on three benchmark datasets.}}

\titleAT 

\newpage
\twocolumn

\title{Robust Estimation of 3D Human Poses from a Single Image}

\author[1,2]{Chunyu Wang}
\author[1,2]{Yizhou Wang}
\author[2]{Zhouchen Lin}
\author[3]{Alan L. Yuille}
\author[1]{Wen Gao}
\affil[1]{Nat'l Engineering Lab for Video Technology, Sch'l of EECS, Peking University, Beijing, China}
\affil[2]{Key Lab. of Machine Perception (MOE), Sch'l of EECS, Peking University, Beijing, China}
\affil[3]{Department of Statistics, University of California, Los Angeles (UCLA), USA}

\maketitle

\begin{abstract}
Human pose estimation is a key step to action recognition. We propose a method of estimating 3D human poses from a single image, which works in conjunction with an existing 2D pose/joint detector. 3D pose estimation is challenging because multiple 3D poses may correspond to the same 2D pose after projection due to the lack of depth information. Moreover, current 2D pose estimators are usually inaccurate which may cause errors in the 3D estimation. We address the challenges in three ways: (i) We represent a 3D pose as a linear combination of a sparse set of bases learned from 3D human skeletons. (ii) We enforce limb length constraints to eliminate anthropomorphically implausible skeletons. (iii) We estimate a 3D pose by minimizing the $L_1$-norm error between the projection of the 3D pose and the corresponding 2D detection. The $L_1$-norm loss term is robust to inaccurate 2D joint estimations. We use the alternating direction method (ADM) to solve the optimization problem efficiently. Our approach outperforms the state-of-the-arts on three benchmark datasets.
\end{abstract}

\section{Introduction}
Action recognition is a key problem in computer vision \cite{wang2011mining} and has many applications such as human-computer interaction and video surveillance. Since an action is naturally represented by human poses \cite{ChunyuCVPR13}, 2D and 3D pose estimation has attracted a lot of attention. A 2D pose is usually represented by a set of joint locations \cite{Yang2D} whose estimation remains challenging because of the huge human appearance variation, viewpoint change, etc.

A 3D pose is typically represented by a skeleton model parameterized by joint locations \cite{Taylor} or by rotation angles \cite{lee2009human}. The representation is intrinsic as it is invariant to viewpoint changes. However, estimating 3D poses from a single image remains a difficult problem. First, a 3D pose is usually inferred from 2D joint locations. So, the accuracy of 2D joint estimation can greatly affect the 3D estimation performance. Second, multiple 3D poses may correspond to the same 2D pose after projection. This introduces severe ambiguities in 3D pose estimation. Third, the problem is further complicated when camera parameters are unknown.

\begin{figure}
\centering
\includegraphics[width=3.2in]{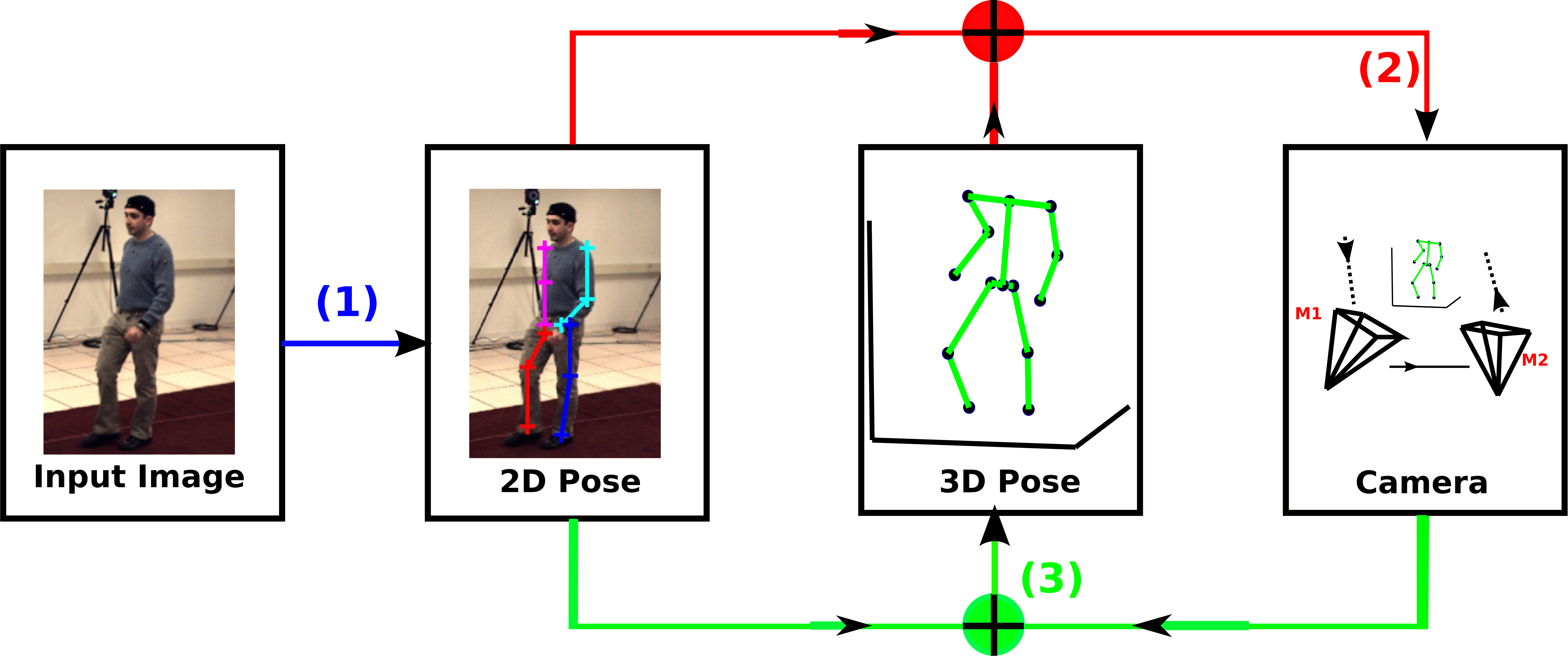}
\caption{\textbf{Method overview}. (1) On a test image, we first estimate the 2D joint locations and initialize a 3D pose. (2) Then camera parameters are estimated from the 2D and 3D poses. (3) Next we update the 3D pose with the current camera parameters and the 2D pose. We repeat steps (2) and (3) until convergence.}
\label{fig:framework}
\end{figure}

We propose a novel method, which alternately updates the 3D pose and camera parameters. Figure \ref{fig:framework} shows the overview of the method. On an input image, we first employ a 2D pose estimator ({\it e.g.} \cite{Yang2D}) to detect the 2D joints. Then we initialize a 3D pose ({\it e.g.} the mean pose). Using both the poses, we estimate the camera parameters (step 2). Next, we re-estimate the 3D pose with the current camera parameters (step 3). Step 2 and 3 are iterated until convergence.

We represent a 3D human pose by a linear combination of a set of overcomplete bases. Since human poses lie in a low dimensional space \cite{elgammal2004inferring}, in the basis pursuit optimization, we enforce an $L_1$-norm regularization on the basis coefficients so that only a few of them are activated. Such holistic representation is able to reduce the ambiguities in the 3D pose estimation and is robust to occlusions ({\it e.g.} missing joints), because it encodes the structural prior of the human skeleton manifold.

We estimate a 3D pose ({\em i.e.} basis coefficients) by minimizing an $L_1$-norm penalty between the projection of the 3D joints and the 2D detections. The commonly used $L_2$-norm tends to distribute errors evenly over all variables. When some joints of the estimated 2D pose are inaccurate, the inferred 3D pose may be biased to a completely wrong configuration. In contrast, $L_1$-norm is more tolerant to the inaccurate 2D joints. However, even if the $L_1$-norm error is adopted, the inferred 3D skeleton may still violate the anthropometric quantities such as limb proportions. Hence, we enforce eight limb length constraints on the estimated 3D pose to eliminate the incorrect ones.

We use an efficient alternating direction method (ADM) to solve our optimization problem. Although global optimality is not guaranteed, we obtain reasonably good solutions. Our method outperforms the state-of-the-arts on three benchmark datasets.

The paper is organized as follows. Section \ref{sec:relatedwork} reviews the related work. Section \ref{sec:approach} introduces the proposed approach. Section \ref{sec:experiment} shows implementation details and experiment results. Conclusion is in section \ref{sec:summary}. Section \ref{sec:optimization} (Appendix) presents the optimization method in detail.

\section{Related Work}
\label{sec:relatedwork}

Existing work on 3D pose estimation can be classified into four categories according to their inputs to the system, {\it e.g.} the image, image features, camera parameters, etc.
The first class \cite{Lee} \cite{SimoSerraCVPR2012} takes camera parameters as inputs. For example, Lee et al. \cite{Lee}
represent a 3D pose by a skeleton model and parameterize the body
parts by truncated cones. They estimate the rotation angles of body parts by minimizing
the silhouette discrepancy between the model projections and the
input image by applying Markov Chain Monte Carlo
(MCMC). Simo-Serra et al. \cite{SimoSerraCVPR2012} represent
a 3D pose by a set of joint locations. They automatically estimate the 2D pose, model each joint by a Gaussian distribution, and propagate the uncertainty to 3D pose space. They
sample a set of 3D skeletons from the space and learn a SVM to determine the most feasible one.

The second class \cite{c6} \cite{c5} requires manually labeled 2D joint locations in a video as input.
Valmadre et al. \cite{c6} first apply structure from motion to
estimate the camera parameters and the 3D pose of the rigid torso,
and then require human input to resolve the depth ambiguities for
non-torso joints. Wei et al. \cite{c5} propose the ``rigid body
constraints'', {\em e.g.} the pelvis, left and right hip joints form a
rigid structure, and require that the distance between any two joints
on the rigid structure remains unchanged across time. They
estimate the 3D poses by minimizing the discrepancy between the
projection of the 3D poses and the 2D joint detections without violating the
``rigid body constraints''.

The third class \cite{Taylor} \cite{Ramakrishna} requires manually labeled 2D joints in one image. Taylor \cite{Taylor} assumes that the limb lengths are known and calculates the relative depths of the limbs by considering foreshortening. It requires human input to resolve the depth ambiguities at each joint. Ramakrishna et al. \cite{Ramakrishna} represent a 3D pose by a linear combination of a set of bases. They split the
training data into classes, apply PCA to each class, and combine
the principal components as bases. They greedily add the most
correlated basis into the model and estimate the coefficients by
minimizing an $L_2$-norm error between the projection of 3D pose
and the 2D pose. They enforce a constraint on the sum of the limb
lengths, which is just a weak constraint. This work \cite{Ramakrishna} achieves the state-of-the-art performance but relies on \emph{manually labeled} 2D joint
locations.

The fourth class \cite{mori2006recovering} \cite{elgammal2004inferring} requires only a single image or image features ({\it e.g.} silhouettes). For example, Mori et al. \cite{mori2006recovering} match a test image to the stored exemplars using shape context descriptors, and transfer the matched 2D pose to the test image. They lift the 2D pose to 3D using the method proposed in \cite{Taylor}. Elgammal et al.\cite{elgammal2004inferring} propose to learn view-based silhouettes manifolds and the mapping function from the manifold to 3D poses. These approaches do not explicitly estimate camera parameters, but require a lot of training data from various viewpoints.

Our method requires only a single image to infer 3D human poses. It is similar to \cite{Ramakrishna}
but there are five distinctive differences. (i) We obtain 2D
joint locations by running a detector \cite{Yang2D} rather
than by manual labeling. (ii) We use $L_1$-norm penalty instead
of the $L_2$-norm one as it is more robust to inaccurate 2D joint locations. (iii) They \cite{Ramakrishna} enforce a weak anthropomorphic constraint
({\em i.e.} sum of limb length) for the sake of computational
simplicity, which is insufficient to eliminate incorrect poses;
while we enforce eight limb length constraints, which is much more
effective. (iv) We enforce an
$L_1$-norm constraint on the basis coefficients rather than
greedily adding bases into the model to encourage sparsity. They
need to re-estimate the coefficients every time a new
basis is introduced, which is inefficient. (v) We use an efficient
alternating direction method to solve our optimization problem.

\section{Our Approach}
\label{sec:approach} We represent 2D and 3D poses by $n$ joint
locations $x \in \mathbb{R}^{2n}$ and $y \in \mathbb{R}^{3n}$,
respectively. By assuming a weak perspective camera model, the 2D projection $x$ of a 3D
pose $y$ in an image are related as: $x=M y$, where $M=I_{n}
\otimes M_0$, in which $I$ is the identity matrix, $\otimes$ is
the
Kronecker product, and $M_0= \begin{pmatrix} m_1^T\\
m_2^T
\end{pmatrix}
\in \mathbb{R}^{2 \times 3}$ is the camera projection matrix.
Given the estimated $x$, we alternately estimate the camera parameter
$M_0$ and the 3D pose $y$. We describe the details for 3D pose estimation in section \ref{sec:3d_pose_sec} and for camera parameter
estimation in section \ref{sec:camera_sec}.

\subsection{Robust 3D Pose Estimation}
\label{sec:3d_pose_sec} We represent a 3D pose $y$ as a linear
combination of a set of bases $B=\{b_1, \cdots, b_k\}$, {\em i.e.}
$y=\sum_{i=1}^k{\alpha_i \cdot b_i} + \mu $, where $\alpha$ are
the basis coefficients and $\mu$ is the mean pose. Given a 2D pose
$x$ and camera parameter $M_0$, we estimate the coefficients
$\alpha$ by minimizing an $L_1$-norm error between the projection of
the estimated 3D pose and the 2D pose: $\left\| M \left( B \alpha
+ \mu \right)-x \right\|_1$. We also enforce $L_1$-norm
regularization on the basis coefficients $\alpha$ and eight limb
length constraints on the inferred 3D pose.

\begin{figure}[th]
\centering
\includegraphics[width=2.3in]{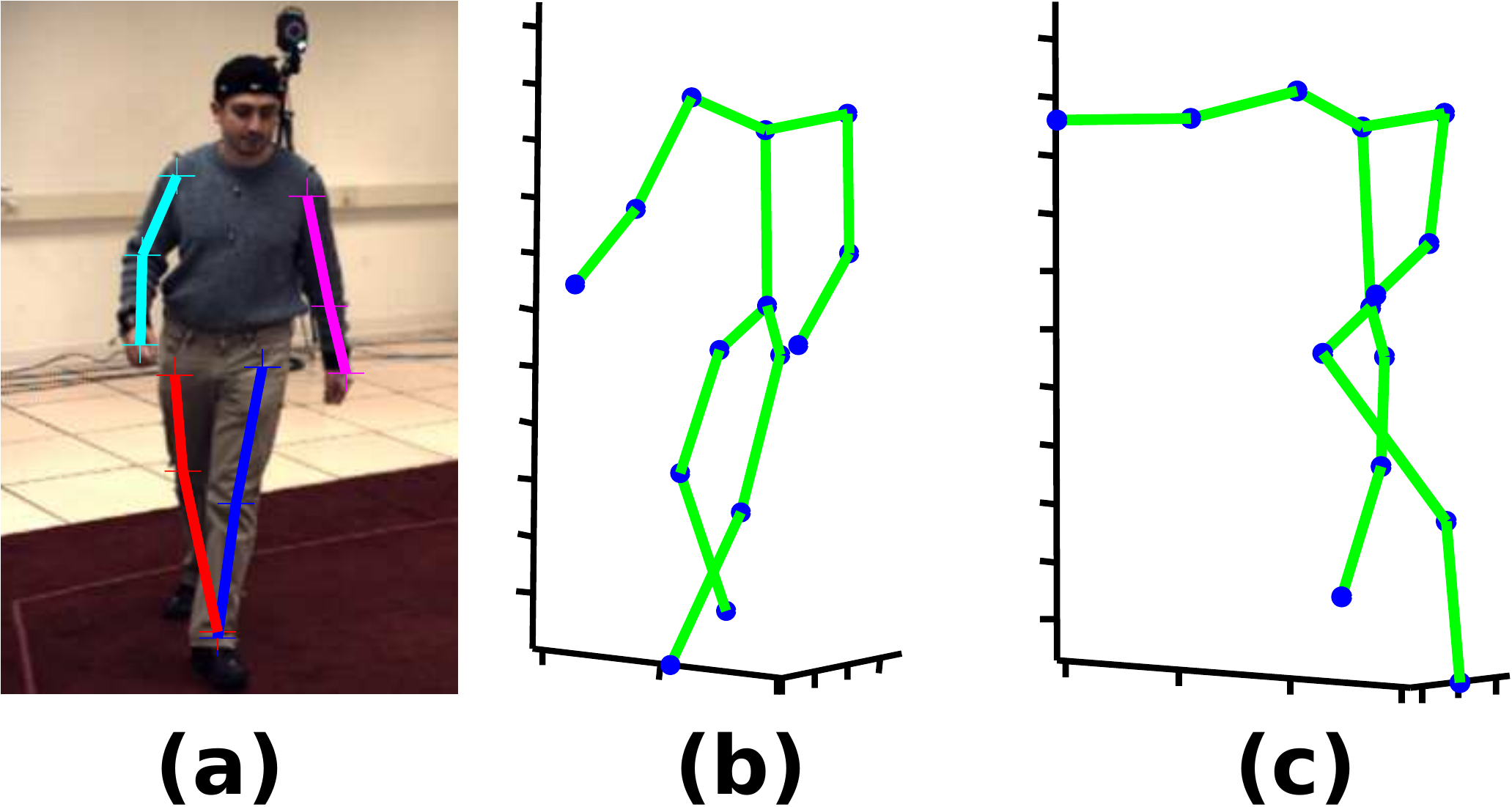}
\caption{Comparison of 3D pose estimation by minimizing $L_1$-norm
vs $L_2$-norm penalty. \textbf{(a)} estimated 2D joint
locations where the right foot location is inaccurate.
\textbf{(b-c)} are the estimated 3D poses using the $L_1$-norm and
$L_2$-norm, respectively. The $L_2$-norm penalty biases the
estimation to a wrong pose.}
\vspace{-1em} \label{fig:L1_L2_sample}
\end{figure}

\subsubsection{$L_1$-norm Objective Function}
\label{sec:l1norm} $L_2$-norm is the most widely used
error metric in the literature. However, it is
sensitive to inaccuracies in 2D pose estimation, which are usually
caused by failures in feature detections and other factors,
because it tends to distribute errors uniformly. In this work, we
propose to minimize an $L_1$-norm error, {\em i.e.} $\left\| x-M (B \alpha
+\mu) \right\|_1$. As a widely used robust regularizer in statistics, the
$L_1$ penalty is robust to inaccurate 2D joint outliers. For
example, in Figure \ref{fig:L1_L2_sample} the 2D location of the
right foot is inaccurate. The estimated 3D pose using $L_2$-norm
error is drastically biased to a wrong configuration. The
camera parameter estimation is also incorrect. However, using
$L_1$-norm returns a reasonable 3D pose. Extensive experiments on
benchmark datasets justify that using the $L_1$-norm can improve
the performance, especially when 2D pose estimation is
inaccurate.

\subsubsection{Sparsity Constraint on the Basis Coefficients}
Although human poses are highly variant, they lie in a
low dimensional space \cite{Safonova}. Hence, we enforce sparsity on
the basis coefficients $\alpha$ so that each 3D pose is
represented by only a few bases. The sparsity can be induced by
minimizing the $L_1$-norm of $\alpha$. This is an important
structural prior to remove incorrect or anthropomorphically
implausible 3D poses. In addition, the sparsity constraint can
also prevent overfitting to (inaccurate) 2D pose estimations. If
there is no sparsity constraint, given a large number of bases we
can always decrease the projection error to zero for an arbitrary
2D pose; however, there is no guarantee that the resulted 3D pose
is correct. In experiments, we observe that the sparsity
constraint is quite important. In summary, the resulted objective function is:

\begin{equation}
\begin{aligned}
& \underset{\alpha}{\text{min}}
& & { \left\| x-M \left(B \alpha +\mu \right) \right\|_1 + \theta \left\|\alpha\right\|_1}
\end{aligned}
\label{eq:objective}
\end{equation}
where $\theta>0$ is a parameter that balances the loss term and
the regularization term.

\subsubsection{Anthropomorphic Constraints}
We require that the eight limb lengths of a 3D pose comply with
certain proportions \cite{c14}. The eight limbs are
left/right-upper/lower-arm/leg. We define a joint selection matrix $E_j=[0, \cdots, I,
\cdots, 0]\in \mathbb{R}^{3 \times 3n}$, where the $j_{th}$ block is
the identity matrix. The product of $E_j$ and $y$ is the 3D
location of the $j_{th}$ joint in pose $y$. Let
$C_i=E_{i_1}-E_{i_2}$. Then $\left\|C_i y\right\|_2^2$ is the
squared length of the $i_{th}$ limb whose ends are the $i_1$-th
and $i_2$-th joints.

We normalize the length of the right lower leg to one and compute
the squared lengths of other limbs (say $L_i$) according to the limb
proportions used in \cite{c14}. The proportions are kept the same for all
people. Now we have constraints $\left\|C_i \left(B \alpha+\mu
\right) \right\|_2^2 = L_i$. Given the camera parameters we can
formulate the 3D pose estimation problem as follows:
\begin{equation}
\begin{aligned}
& \underset{\alpha}{\text{min}}
& & { \left\| x-M \left(B \alpha +\mu \right) \right\|_1 + \theta \left\| \alpha \right\| _1 } \\
& \text{s.t.}
& & \left\| C_i \left(B \alpha+\mu \right) \right\|_2^2 = L_i, i=1,\cdots,t
\end{aligned}
\label{eq:final}
\end{equation}

\subsection{Robust Camera Parameter Estimation}
\label{sec:camera_sec} Given a 3D pose, we estimate the camera
parameters by minimizing the $L_1$-norm projection error. We
reshape the 2D and 3D poses, $x$ and $y$, as $X \in \mathbb{R}^{2
\times n}$ and $Y \in \mathbb{R}^{3 \times n}$, respectively. Then
ideally $X=M_0Y$ should hold, where $ M_0=\left(
\begin{array}{c}m_1^T
\\ m_2^T\end{array}\right)$ is the projection matrix of a weak
projective camera, {\em i.e.} $m_1^Tm_2=0$. Due to errors, we estimate the camera parameters $m_1$ and $m_2$ by solving the
following problem:
\begin{equation}
\underset{m_1, m_2}{\text{min}} \left\| X-\left(
\begin{array}{c}
m_1^T \\
m_2^T
\end{array} \right)Y \right\|_1 ,\quad \text{s.t.}
\quad m_1^Tm_2=0. \label{eq:camera}
\end{equation}

\subsection{Optimization}
We alternately update the 3D pose and the camera parameters. We
first initialize the 3D pose $X$ by the mean pose of the
training data, and estimate camera parameters $m_1$ and $m_2$ by
solving problem (\ref{eq:camera}). With the updated camera
parameters, we then re-estimate the 3D pose by solving problem
(\ref{eq:final}). We repeat the above process until convergence or
the maximum number of iterations is reached. We use the
alternating direction method to solve the two optimization
problems efficiently. Please see Appendix for details.

\section{The Experimental Results}
\label{sec:experiment} We conduct two types of experiments to
evaluate our approach. The first type is controlled. We assume
that the 2D joint locations are known and evaluate
the influence: (i) of the \textbf{three factors} ({\em i.e.} the sparsity term,
the anthropomorphic constraints and the $L_1$-norm penalty), (ii) of
the inaccurate 2D pose estimations and (iii) of the human-camera angles, on
the 3D pose estimation performance. The second type is real. We
estimate the 2D pose in an image by a detector \cite{Yang2D} and then estimate the 3D skeletons. We compare our
method with the state-of-the-art ones \cite{Ramakrishna}
\cite{SimoSerraCVPR2012} \cite{Daubney}. Our approach can also
refine the 2D pose estimation by projecting the estimated 3D pose
to 2D image.

\begin{figure}
\subfigure[]{
\includegraphics[width=1.55in]{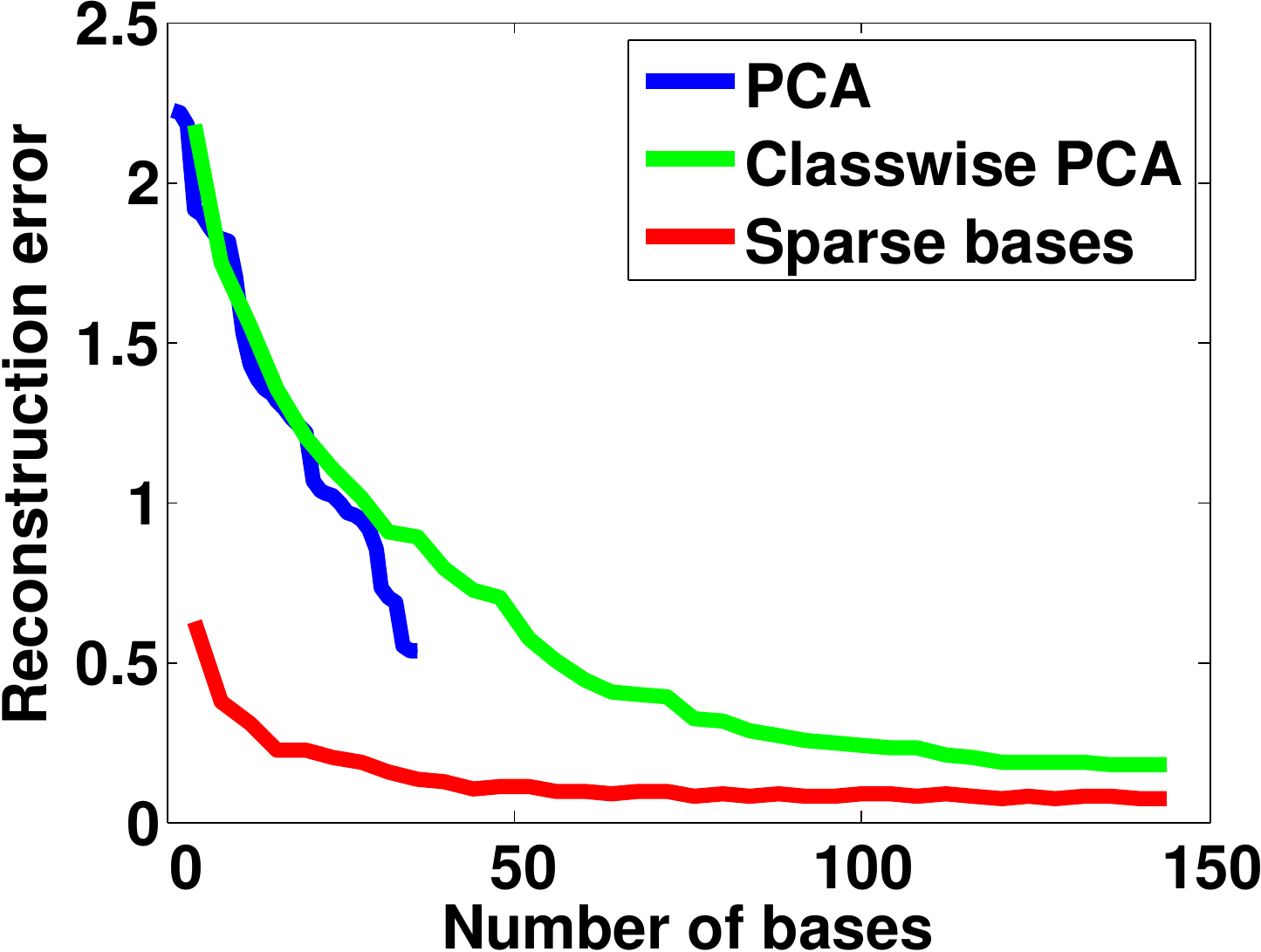}}
\subfigure[]{
\includegraphics[width=1.55in]{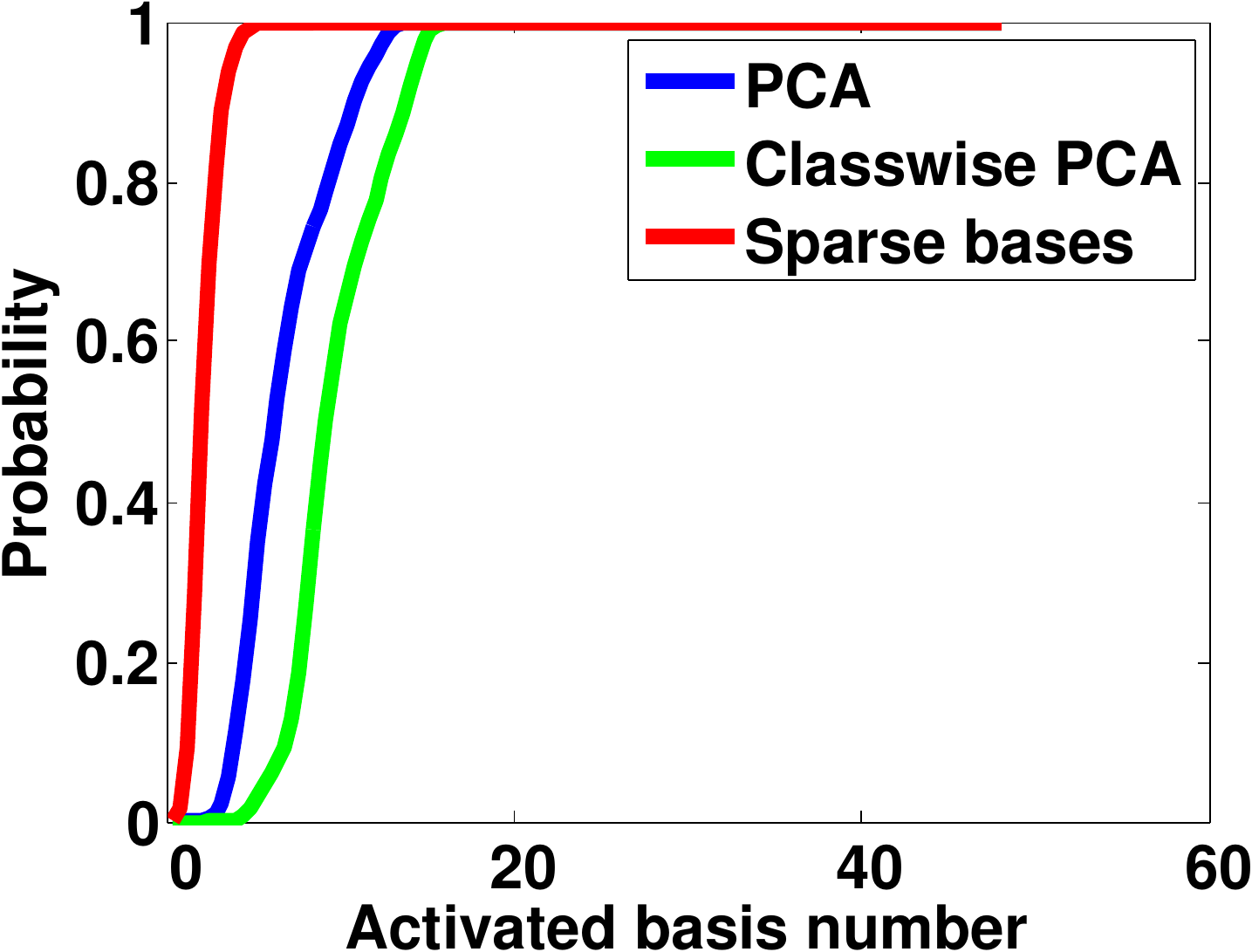}}
\centering \caption{Comparison of the three basis learning methods
on the CMU dataset. \textbf{(a)} 3D pose reconstruction errors
using different number of bases. \textbf{(b)} Cumulative distribution of the number of activated bases in represent the 3D poses. The y-axis is the percentage of the cases whose activated basis number is less than or equal to the corresponding x-axis value on the curves.}
\label{fig:basis}
\end{figure}

We use 12 body joints, {\em i.e.} the left and right shoulders, elbows,
hands, hips, knees and feet, being consistent with the 2D
pose detector \cite{Yang2D}. $200$ bases are used for all
experiments and about $6$ of them are activated for representing a 3D pose. In optimization, we terminate the
algorithm if the number of iterations exceeds $20$.

\subsection{The Datasets}
We evaluate our approach on three datasets: the CMU motion dataset
\cite{c10}, the HumanEva dataset \cite{sigal2006humaneva} and the
UvA 3D pose dataset \cite{c28}. For the CMU dataset, we learn the bases on actions of
``climb'', ``swing'', ``sit'' and ``jump'', and test on different
actions of ``walk'', ``run'', ``golf'' and ``punch'' to justify
the generalization ability of our method. For the HumanEva
dataset, we use the walking and jogging actions of three subjects
for evaluation as in \cite{SimoSerraCVPR2012}. For the UvA dataset, we use the first four sequences for
training and the remaining eight for testing.

\begin{figure}
\centering
\includegraphics[width=2.2in]{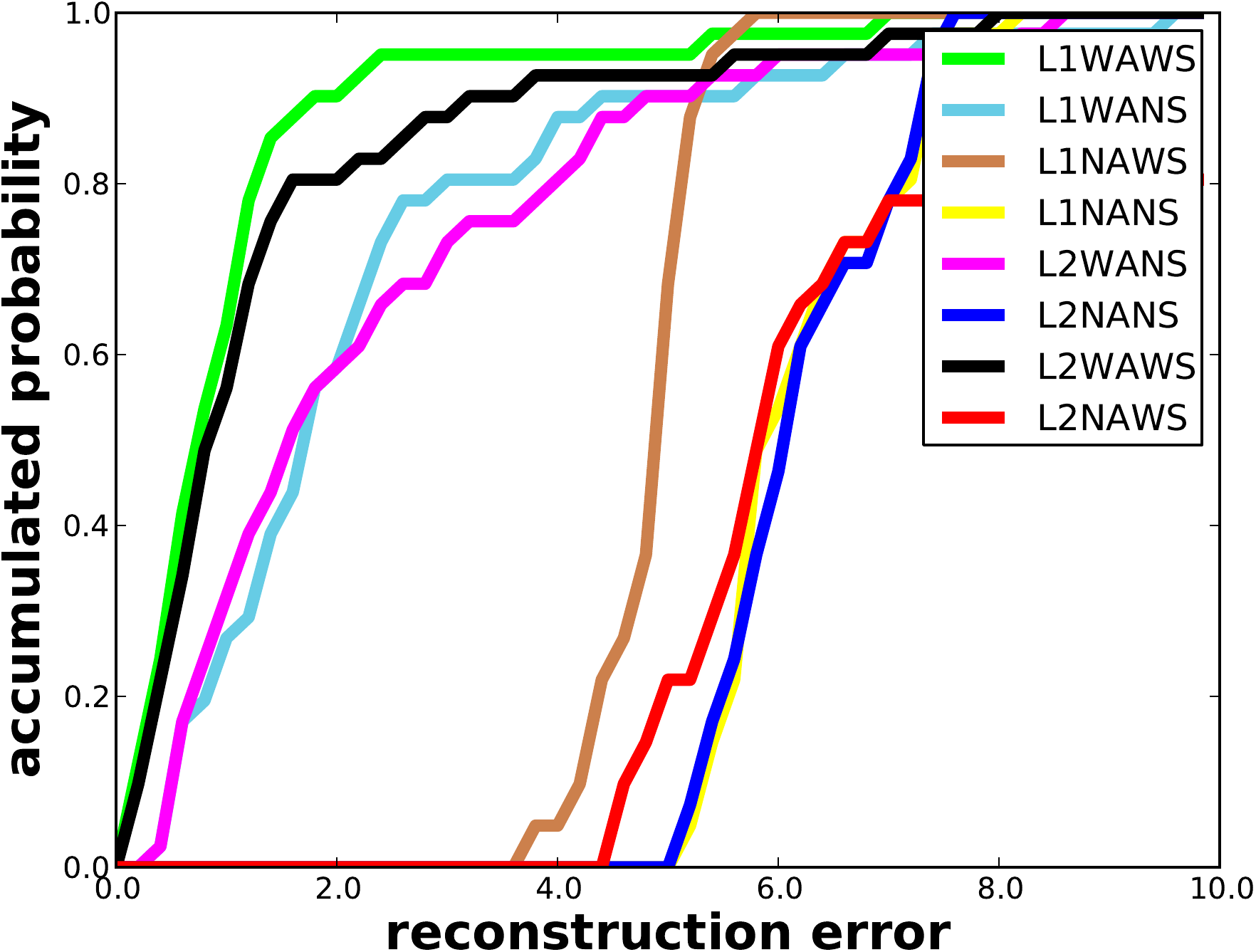}
\caption{\textbf{Controlled experiment:} Cumulative distribution
of 3D pose estimation errors on the CMU dataset. The y-axis is the percentage of the cases whose estimation error is less than or equal to the corresponding x-axis value on the curves.} \vspace{-1em} \label{fig:CMU_UVA}
\end{figure}

\subsection{Basis Learning}
\label{sec:basislearning} Our approach pursues a set of sparse
bases by enforcing an $L_1$-norm regularization on the basis
coefficients (as in \cite{MairalSparse}). But we also compare with
other two basis learning methods. The first method applies principle component analysis (PCA) on the training set and uses the first $k$ principal components as the bases. The second splits the training set into
different classes by action labels, then applies PCA on each class,
and finally collect the principal components as bases (which we
call classwise PCA) as in \cite{Ramakrishna}.

We learn the bases on the training data (of
four action classes) of the CMU dataset, and reconstruct each test
3D pose by solving an $L_1$-norm regularized least square problem.
The reconstruction errors are shown in Figure \ref{fig:basis} (a).
The sparse bases consistently achieve the
lowest errors among the three methods. Note that the maximum
number of bases for PCA and classwise PCA is $36$ (i.e. the
dimension of a 3D pose) and $144$, respectively. In addition,
fewer bases are activated using the $L_1$-norm induced bases (see Figure \ref{fig:basis}
(b)). This
justifies the bases' representative power.

\subsection{Controlled Experiments}
We assume the ground-truth 2D pose $x$ is known and recover the 3D pose
$y$ from $x$. The residual error between the estimated 3D pose
$\hat{y}$ and the ground truth $y$, {\em i.e.} $||y-\hat{y}||_2$, is
used as the evaluation criterion as in \cite{Ramakrishna}.

\subsubsection{Influence of the Three Factors}
\label{sec:threefactors} We design seven baselines to evaluate the
influence of the three factors, {\em i.e.} the sparsity term, the
anthropomorphic constraints and the $L_1$-norm penalty. The first
baseline is symbolized as L2NAWS, {\em i.e.} the approach uses an $L_2$-norm objective
function, No Anthropomorphic constraints and With the Sparsity
constraint. The remaining baselines are L2NANS, L2WANS, L2WAWS,
L1NANS, L1NAWS and L1WANS, which can be similarly understood by
their names. We solve the optimization problem in L2WANS and
L2WAWS by the alternating direction method. The optimization
problems in other baselines can be solved trivially.

Figure \ref{fig:CMU_UVA} shows the results on the CMU dataset.
First, the baselines without the sparsity term perform worse than
those with the sparsity term. Second, enforcing limb length constraints improves the
performance ({\it e.g.} L2WAWS outperforms L2NAWS). Third, $L_1$-norm
outperforms $L_2$-norm ({\it e.g.} L1NAWS is better than L2NAWS). Finally, our
approach performs best among the baselines.

\subsubsection{Influence of Inaccurate 2D Poses}
\label{sec:influenc_inaccurate} We evaluate the robustness of our
approach against inaccurate 2D pose estimations. We synthesize noisy 2D poses by generating ten
levels of random Gaussian noises and adding them to the original 2D poses. The
magnitude of the tenth (largest) level noise is one, which is equal to the normalized length of the right lower leg.
We estimate the 3D poses from those corrupted 2D joints.

Figure \ref{fig:noise_CMU} shows the results. First, L1NANS
outperforms L2NANS, which demonstrates that $L_1$-norm is more
robust to 2D pose errors. Second, L2NANS and L2WANS get larger
errors than L2NAWS and L2WAWS, respectively, which shows the
importance of sparsity in handling inaccurate 2D poses. Our
approach achieves a better performance than all baselines and
Ramakrishna et al's method \cite{Ramakrishna}.

\begin{table}
\caption{\textbf{Real experiment on the HumanEva dataset:}
comparison with the state-of-the-art methods
\cite{SimoSerraCVPR2012} \cite{Daubney}. We present results for
both walking and jogging actions of all three subjects and camera
C1. The numbers in each cell are the root mean square error (RMS)
and standard deviation, respectively. We use the unit of
millimeter as in \cite{SimoSerraCVPR2012} and \cite{Daubney}. The
length of the right lower leg is about $380$ mm. See Section
\ref{sec:realexperiment}.} \centering
\begin{tabular}{|c|c|c|c|}
\hline
Walking & S1 & S2 & S3 \\
\hline
Ours & 71.9 (19.0) & 75.7 (15.9) & 85.3 (10.3) \\
\hline
\cite{SimoSerraCVPR2012} & 99.6 (42.6) & 108.3 (42.3) & 127.4 (24.0) \\
\hline
\cite{Daubney} & 89.3 & 108.7 & 113.5 \\
\hline

\hline
\hline
Jogging & S1 & S2 & S3 \\
\hline
Ours & 62.6 (10.2) & 77.7 (12.1) & 54.4 (9.0) \\
\hline
\cite{SimoSerraCVPR2012} & 109.2 (41.5) & 93.1 (41.1) & 115.8 (40.6) \\
\hline
\end{tabular}
\label{table:stateofart}
\end{table}

\subsubsection{Influence of Human-Camera Angles}
\label{sec:humancamera} We explore the influence of human-camera
angles on 3D pose estimation. We first transform
the 3D poses into a local coordinate system, where the x-axis is
defined by the line passing the two hips, the y-axis is defined by
the line of spine and the z-axis is the cross product of the
x-axis and y-axis. Then we rotate the 3D poses around y-axis by a
particular angle, ranging from 0 to 180, and project them to 2D by a weak perspective camera. We estimate the 3D poses
from their 2D projections. Figure \ref{fig:camera_angle} shows
that the estimation errors using \cite{Ramakrishna} increase
drastically as human moves from profile (90 degrees) towards
frontal pose (0 degree). This may be due to the fact that frontal view has more severe foreshortenings than the profile view, hence introduces more ambiguities into 3D pose estimation. Our approach is more
robust against viewpoint changes.

\begin{table*}
\caption{\textbf{Real experiment on the UvA dataset:} Comparison
of 2D pose estimation results. We report: (1) the Probability of
Correct Pose (PCP) for the eight body parts ({\em i.e.} left upper arm
(LUA), left lower arm (LLA), right upper arm (RUA), right lower
arm (RLA), left upper leg (LUL), left lower leg (LLL), right upper
leg (RUL) and right lower leg (RLL)), (2) PCP for the whole pose,
(3) and the Euclidean distance between the estimated 2D pose and
the groundtruth in pixels.} \centering
\begin{tabular}{|c|c|c|c|c|c|c|c|c|c||c|}
\hline
\multirow{2}{*}{} & \multicolumn{9}{|c|}{ PCP} & \multirow{2}{*}{ Pixel Diff.} \\
\cline{2-10}
 & LUA & LLA & RUA & RLA & LUL & LLL & RUL & RLL & Overall &  \\
\hline
Yang et al. \cite{Yang2D} & 0.751  &\textbf{0.416} & 0.771 & \textbf{0.286} & 0.857 & 0.825 & 0.910 & 0.894 & 0.714 & 109\\
\hline
Ramakrishna et al. \cite{Ramakrishna} & 0.792          & 0.383 & 0.722 & 0.241 & 0.906 & 0.829 & 0.890 & 0.849 & 0.702 & 62\\
\hline
Ours                  & \textbf{0.829} & 0.376 & \textbf{0.800} & 0.245 & \textbf{0.955} & \textbf{0.861} & \textbf{0.963} & \textbf{0.902} & \textbf{0.741} & \textbf{55}\\
\hline
\end{tabular}
\label{table:2dpose}
\end{table*}

\begin{figure}
\centering
\includegraphics[width=2in]{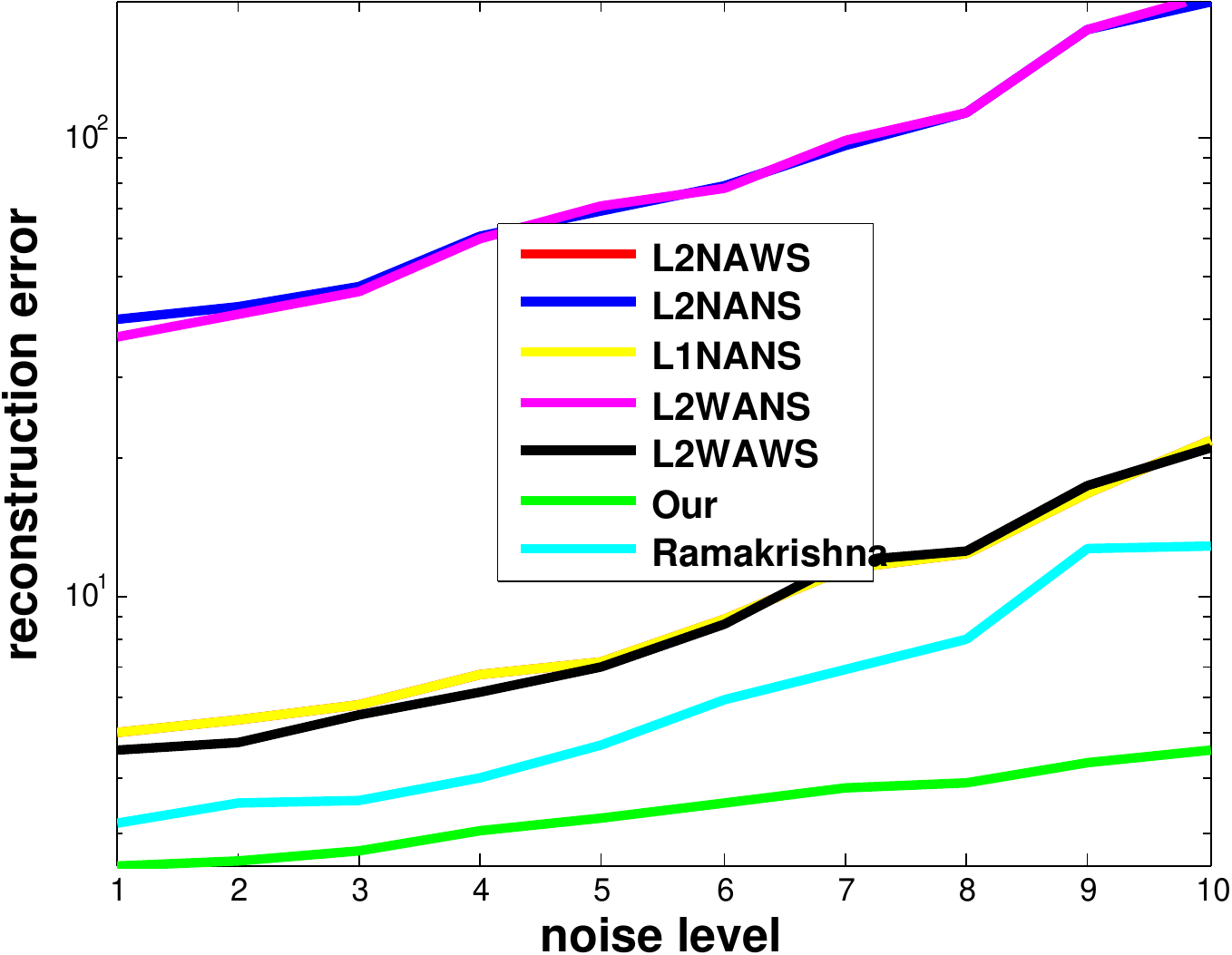}
\caption{\textbf{Controlled experiment on the CMU dataset:} 3D
pose estimation errors when different levels of noises are added
to 2D poses. See Section \ref{sec:influenc_inaccurate}.}
\vspace{-1em} \label{fig:noise_CMU}
\end{figure}

\begin{figure}
\centering
\includegraphics[width=2in]{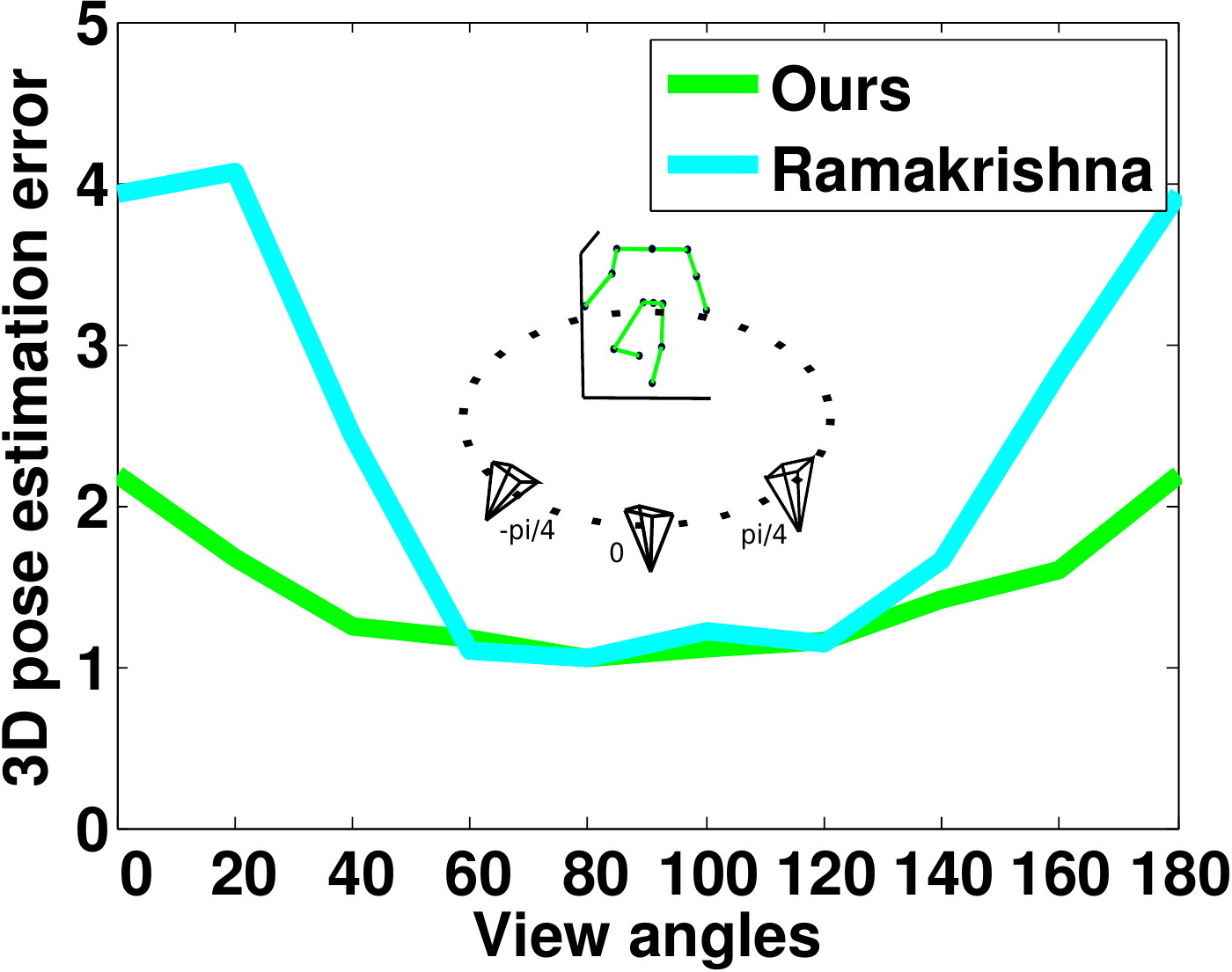}
\caption{\textbf{Controlled experiment on the CMU dataset:} 3D
pose estimation error when the human-camera angle varies from 0 to
180. See Section \ref{sec:humancamera}.} \vspace{-1em}
\label{fig:camera_angle}
\end{figure}

\subsection{Real Experiments}
Given an image, we first detect the 2D joint locations by a detector
\cite{Yang2D}, from which we estimate the corresponding 3D pose
using the proposed approach.

\subsubsection{Comparisons to the State-of-the-arts}
\label{sec:realexperiment} We compare our 3D pose estimator
against a state-of-the-art method \cite{Ramakrishna} on the UvA
dataset. Figure \ref{fig:CMU_UVA_REAL} shows the estimation errors
on each joint. Our approach achieves smaller estimation errors on
all joints, especially for the left and the right hands. This
proves that our approach is robust to inaccurate 2D joint
locations. We also compare our approach with the state-of-the-arts \cite{SimoSerraCVPR2012} \cite{Daubney} on the HumanEva
dataset. Table \ref{table:stateofart} shows the root mean square
errors adopted in \cite{SimoSerraCVPR2012}. Our approach
outperforms both \cite{SimoSerraCVPR2012} and \cite{Daubney}.

\subsubsection{Evaluation on Camera Parameter Estimation}
\label{sec:exp_camera} Our camera parameter estimation usually
converges within nine iterations. Figure \ref{fig:camera} shows
the 3D pose estimation results using the estimated cameras and
groundtruth cameras, respectively. We can see
that the difference is subtle for $70\%$ of cases. We discover
that the initialization of the 3D pose can influence the
estimation accuracy. So we cluster the training poses into $30$
clusters and initialize the 3D pose with each of the cluster
centers for parallel optimization. We keep the one with the
smallest error. We see that the performance can be further improved.

\begin{figure}
\centering
\includegraphics[width=1.9in]{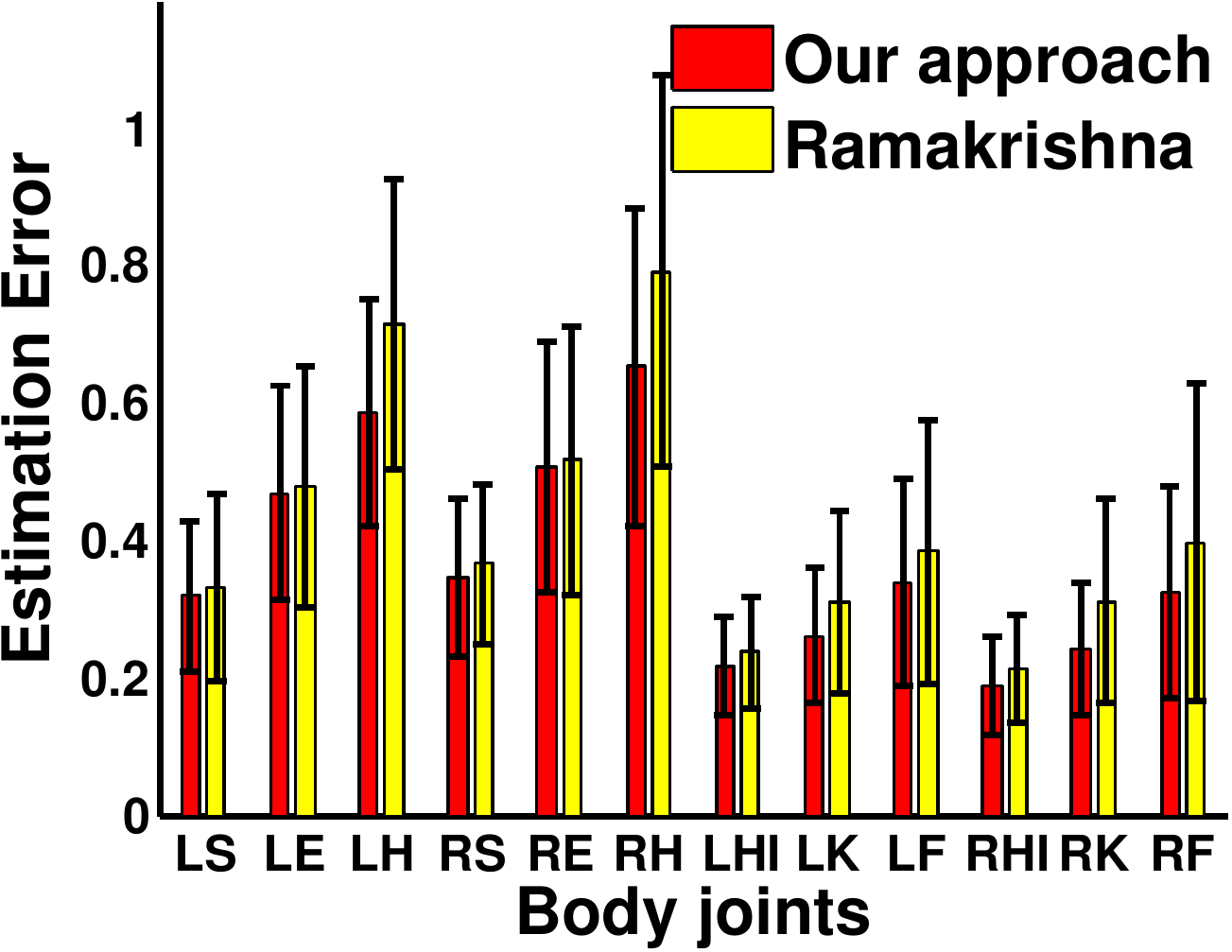}
\caption{\textbf{Real experiment on the UvA dataset:} comparison
with a state-of-the-art \cite{Ramakrishna}. Both average
estimation errors and standard deviations are shown for each joint
({\em i.e.} left shoulder, left elbow, left hand, right shoulder, right
elbow, right hand, left hip, left knee, left foot, right hip,
right knee and right foot). See Section \ref{sec:realexperiment}.}
\label{fig:CMU_UVA_REAL}
\end{figure}

\subsubsection{Evaluation on 2D Pose Estimation}
We project the estimated 3D pose to 2D and compare with the
original 2D estimation \cite{Yang2D}. We report the results
using two criteria. The first is the probability of correct pose
(PCP) \cite{Yang2D} --- an estimated body part is considered correct if its
segment endpoints lie within 50\% of the length of the
ground-truth segment from their annotated location.
The second criterion is the Euclidean distance between the
estimated 2D pose and the groundtruth in pixels as in
\cite{SimoSerraCVPR2012}. Table \ref{table:2dpose} shows that our
approach performs the best on six body parts. In particular, we
improve over the original 2D pose estimators by about $0.03$
(0.741 vs. 0.714) using the first PCP criteria. Our approach also performs the best using the
second criterion.

\section{Conclusion}
\label{sec:summary} We address the problem of estimating 3D human
poses from a single image. The approach is used in conjunction
with an existing 2D pose detector. It is robust to inaccurate 2D pose
estimations by using a sparse basis representation,
anthropometric constraints and an $L_1$-norm projection error
metric. We use an efficient alternating direction method to solve
the optimization problem. Our approach outperforms the
state-of-the art ones on three benchmark datasets.

\textbf{Acknowledgements:}We'd like to thank for the support from the following research grants NSFC-61272027, NSFC-61231010, NSFC-61121002, NSFC-61210005 and USA ARO Proposal 62250-CS. And, this material is based upon work supported by the Center for Minds, Brains and Machines (CBMM), funded by NSF STC award CCF-1231216. Z. Lin is supported by NSF China (grant nos. 61272341, 61231002, 61121002) and MSRA.
\begin{figure}
\centering
\includegraphics[width=1.9in]{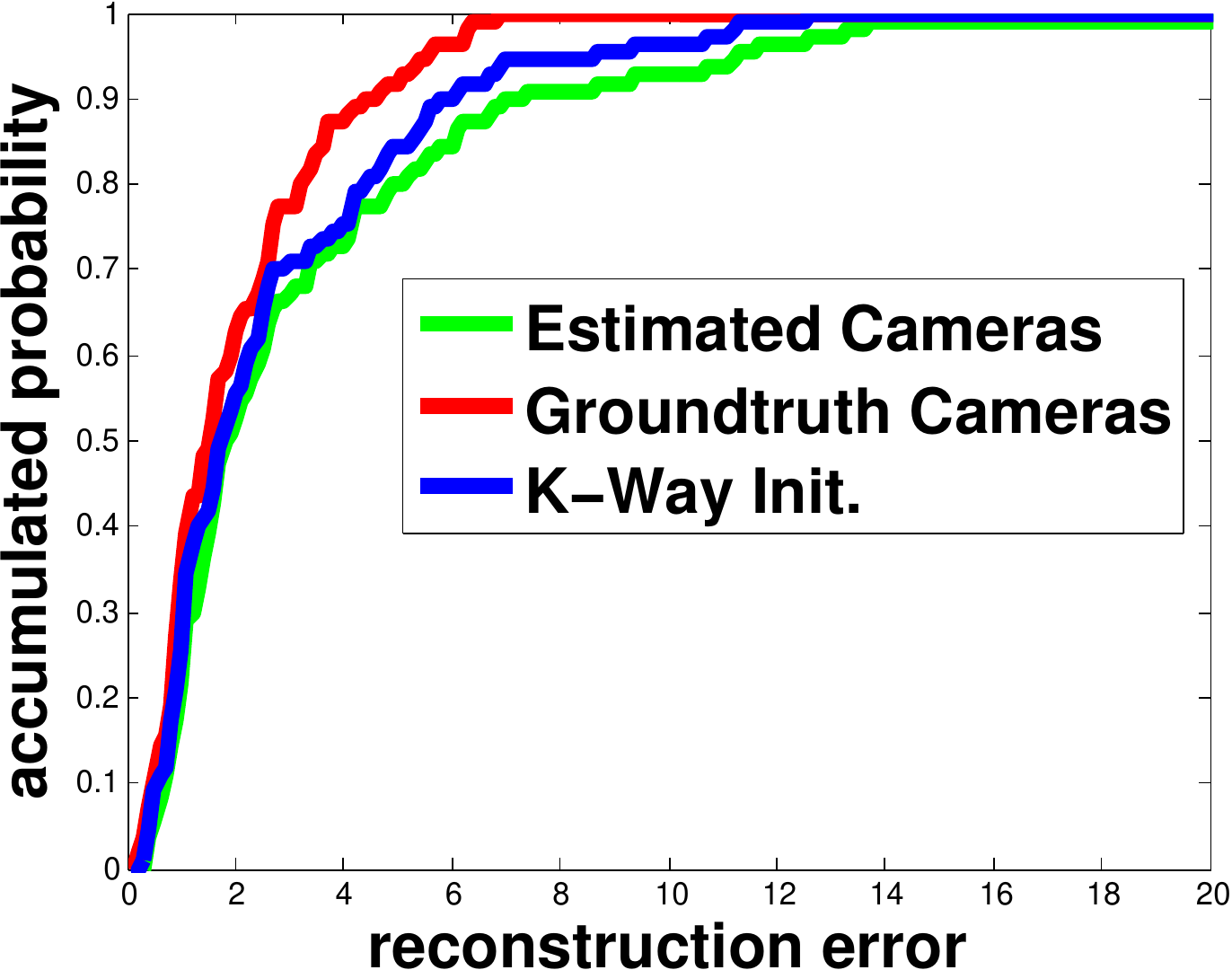}
\caption{\textbf{Real experiment on the CMU dataset:} cumulative
distribution of 3D pose estimation errors when camera parameters
are (1) assigned by groundtruth, estimated by initializing the 3D
pose with (2) mean pose, or (3) $30$ cluster centers. The y-axis is the percentage of the cases whose estimation error is less than or equal to the corresponding x-axis value on the curves. } \vspace{-1em} \label{fig:camera}
\end{figure}

{\small
\bibliographystyle{ieee}
\bibliography{egpaper_final}
}

\section{Appendix: Optimization by ADM}
\label{sec:optimization} Due to space limit, we only sketch the
major steps of ADM for our optimization problems. In the
following, $k$ and $l$ are the number of iterations.

\subsection{3D Pose Estimation}
We introduce two auxiliary variables $\beta$ and $\gamma$ and rewrite Eq.~(\ref{eq:final}) as:
\begin{eqnarray}
\begin{array}{rl}
\underset{\alpha,\beta, \gamma}{\text{min}} & {\left\| \gamma \right\|_1 + \theta \left\| \beta \right\|_1 }\\
\text{s.t.} & \gamma = x-M \left(B \alpha +\mu \right) , \quad \alpha = \beta,  \\
& \left\| C_i \left(B \alpha+\mu \right) \right\|_2^2 = L_i,
i=1,\cdots,m.
\end{array}
\label{eq:adm}
\end{eqnarray}

The augmented Lagrangian function of  Eq.~(\ref{eq:adm}) is:
\begin{eqnarray*}
\begin{array}{rl}
\mathcal{L}_{1}(\alpha, \beta, \gamma, \lambda_1, \lambda_2,\eta)=
|| \gamma
||_1 + \theta || \beta ||_1 + \\
\lambda_1^T{[\gamma-x+M (B
\alpha+\mu)]}+\lambda_2^T{(\alpha-\beta)}+ \\
\frac{\eta}{2}{
\left[||\gamma-x+M (B \alpha +\mu)||^2+||\alpha-\beta||^2 \right]}
\end{array}
\label{eq:lagrange}
\end{eqnarray*}
where $\lambda_1$ and $\lambda_2$ are the Lagrange multipliers and
$\eta>0$ is the penalty parameter. ADM is to update the variables
by minimizing the augmented Lagrangian function w.r.t. the
variables alternately.

\subsubsection{Update $\gamma$}
We discard the terms in $\mathcal{L}_1$ which are independent of
$\gamma$ and update $\gamma$ by:
\begin{equation*}
\gamma^{k+1} = \underset{\gamma}{\text{argmin}}{\left\| \gamma
\right\|_1+\frac{\eta_k}{2} \left\| \gamma-\left[x-M (B \alpha^k +
\mu) - \frac{\lambda_1^k}{\eta_k}\right] \right\| ^2}
\label{eq:update_gamma}
\end{equation*}
which has a closed form solution \cite{liu2013linearized}.

\subsubsection{Update $\beta$}
We drop the terms in $\mathcal{L}_1$ which are independent of
$\beta$ and update $\beta$ by:
\begin{equation*}
\beta^{k+1} = \underset{\beta}{\text{argmin}}{\left\| \beta
\right\|_1+\frac{\eta_k}{2 \theta} \left\|
\beta-\left(\frac{\lambda_2^k}{\eta_k}+\alpha^k\right) \right\|^2}
\label{eq:update_beta}
\end{equation*}
which also has a closed form solution \cite{liu2013linearized}.

\subsubsection{Update $\alpha$}
We dismiss the terms in $\mathcal{L}_1$ which are independent of
$\alpha$ and update $\alpha$ by:
\begin{eqnarray}
\begin{array}{rl}
\alpha^{k+1} =  \arg\min\limits_{\alpha}
& {z^T W z } \\
\mbox{s.t.} & z^T \Omega_i z = 0, \quad i=1,\cdots,m
\end{array}\label{eq:update_alpha}
\end{eqnarray}
where  $z=[\alpha^T \quad 1]^T$, \\
$W\!\!=\!\!\left(  \begin{array}{cc}
B^T M^T M B + I & 0 \\
\!\!\!\!2 \left[ \left( \gamma^{k+1}-x+M \mu +
\frac{\lambda_1^k}{\eta_k} \right)^T M B - \beta^{k+1} +
\frac{\lambda_2^k}{\eta_k} \right] & 0
\end{array} \right)$
and $\Omega_i=\left( \begin{array}{cc}
B^T C_i^T C_i B & B^T C_i^T C_i \mu \\
\mu^T C_i^T C_i B & \mu^T C_i^T C_i \mu - L_i
\end{array} \right)$.

Let $Q=z z^T$. Then the objective function becomes $z^T W
z=\mbox{tr}(WQ)$ and Eq.~(\ref{eq:update_alpha}) is transformed
to:
\begin{eqnarray}
\begin{array}{rl}
\underset{Q}{\text{min}}
& {\mbox{tr}(W Q)} \\
\mbox{s.t.}
& \mbox{tr}(\Omega_i Q) = 0, \quad i=1,\cdots,m, \\
& Q \succeq 0, \quad \mbox{rank}(Q) \leq 1.
\end{array}\label{eq:update_alpha2}
\end{eqnarray}

We still solve problem \eqref{eq:update_alpha2} by the alternating
direction method \cite{liu2013linearized}. We introduce an
auxiliary variable $P$ and rewrite the problem as:
\begin{eqnarray}
\begin{array}{rl}
\underset{Q, P}{\text{min}}
& {\mbox{tr}(W Q)} \\
\text{s.t.} & \mbox{tr}(\Omega_i Q) = 0, \quad i=1,\cdots,m, \\
& P = Q, \quad \mbox{rank}(P) \leq 1, \quad P \succeq 0.
\end{array}
\label{eq:update_alpha_aux}
\end{eqnarray}
Its augmented Lagrangian function is:
\begin{equation*}
\mathcal{L}_{2}(Q,P,G,\delta)=\mbox{tr}(W Q) + \mbox{tr}(G^T(Q-P))
+ \frac{\delta}{2}{\left\| Q-P \right\|_F^2 }
\label{eq:update_alpha_aux2}
\end{equation*}
where $G$ is the Lagrange Multiplier and $\delta>0$ is the penalty
parameter. We update $Q$ and $P$ alternately.

\begin{itemize}
\item {Update $Q$: }
\begin{equation}
Q^{l+1}=\argmin\limits_{
\begin{array}{c}
\mbox{tr}(\Omega_i Q)=0,\\
i=1,\cdots,m
\end{array}
} \mathcal{L}_{2}(Q,P^l,G^l,\delta_l). \label{eq:update_T}
\end{equation}
This is convex and solved using CVX \cite{c23}, a package for specifying and solving convex programs.

\item{Update $P$: } We discard the terms in $\mathcal{L}_{2}$
which are independent of $P$ and update $P$ by:
\begin{equation}
P^{l+1}=\argmin\limits_{
\begin{array}{c}
P \succeq 0,\\
\mbox{rank}(P) \leq 1
\end{array}
} \left\|P-\tilde{Q}\right\|_F^2 \label{eq:update_Y}
\end{equation}
where $\tilde{Q}=Q^{l+1}+\frac{2}{\delta_l}{G^l}$. Note that $
\left\|P-\tilde{Q} \right\|_F^2$ is equal to $\left\| P-
\frac{\tilde{Q}^T+\tilde{Q}}{2} \right\|_F^2$. Then
\eqref{eq:update_Y} has a closed form solution by the following
lemma.
\end{itemize}

\begin{lemma}
\label{lemma1} The solution to
\begin{equation}
\underset{P}{\text{min}} {||P-S||_F^2} \quad \text{s.t.} \quad P
\succeq 0, \quad \mbox{rank}(P) \leq 1 \label{eq:update_Y'}
\end{equation}
is $P=\max(\zeta_1,0) \nu_1 \nu_1^T$, where $S$ is a symmetric
matrix and $\zeta_1$ and $\nu_1$ are the largest eigenvalue and
eigenvector of $S$, respectively.
\end{lemma}

\begin{proof}
Since $P$ is a symmetric semi-positive definite matrix and its
rank is one, we can write $P$ as: $P=\zeta \nu \nu^T$, where
$\zeta \geq 0$. Let the largest eigenvalue of $S$ be $\zeta_1$,
then we have $\nu^T S \nu \leq \zeta_1$, $\forall \nu$. Then we
have:
\begin{eqnarray}
\begin{array}{rl}
||P-S||_F^2 &= ||P||_F^2 + ||S||_F^2 - 2 \mbox{tr}(P^TS) \\
& \geq \zeta^2 + \sum_{i=1}^n{\zeta_i^2}-2 \zeta \zeta_1 \\
&=(\zeta-\zeta_1)^2 + \sum_{i=2}^n{\zeta_i^2} \\
& \geq \sum_{i=2}^n{\zeta_i^2} +  \min(\zeta_1, 0)^2
\end{array}
\end{eqnarray}
\end{proof}
The minimum value can be achieved when $\zeta=\max(\zeta_1,0)$ and
$\nu=\nu_1$.

\subsection{Camera Parameter Estimation}
We introduce an auxiliary variable $R$ and rewrite
Eq.~(\ref{eq:camera}):
\begin{eqnarray}
\begin{array}{rl}
\underset{R,m_1,m_2}{\text{min}}
& {\left\|R\right\|_1} \\
\text{s.t.} & R=X-\left( \begin{array}{c}m_1^T \\ m_2^T
\end{array} \right) Y, \quad m_1^Tm_2=0.
\end{array}
\label{eq:camera_objective}
\end{eqnarray}

We still use ADM to solve problem (\ref{eq:camera_objective}). Its
augmented Lagrangian function is:
\begin{eqnarray*}
\begin{array}{rl}
&\mathcal{L}_{3}(R,m_1,m_2,H,\zeta,\tau)\\
=&\left\| R \right\|_1 + \mbox{tr}\left(H^T \left[ \left(
\begin{array}{c}m_1^T \\ m_2^T
\end{array}\right) Y+R-X \right]\right)+ \zeta (m_1^T m_2) \\
&+ \frac{\tau}{2}{ \left[ \left\| \left( \begin{array}{c}m_1^T \\
m_2^T \end{array} \right) Y+R-X \right\|_F^2 +
\left(m_1^Tm_2\right)^2 \right]}
\end{array}\label{eq:lagrange3}
\end{eqnarray*}
where $H$ and $\zeta$ are Lagrange multipliers and $\tau>0$ is the
penalty parameter.

\subsubsection{Update $R$}
We discard the terms in $\mathcal{L}_{3}$ which are independent of
$R$ and update $R$ by:
\begin{equation}
R^{k+1} = \underset{R}{\text{argmin}}{\left\| R
\right\|_1+\frac{\tau_k}{2}{ \left\| R +  \left(
\begin{array}{c} \left(m_1^k\right)^T \\ \left(m_2^k \right)^T
\end{array} \right)Y - X +
\frac{H^k}{\tau_k}\right\|_F^2}} \\ \nonumber \label{eq:updateE}
\end{equation}
which has a closed form solution \cite{liu2013linearized}.
\subsubsection{Update $m_1$}
We discard the terms in $\mathcal{L}_{3}$ which are independent of
$m_1$ and update $m_1$ by:
\begin{eqnarray}
\begin{array}{c}
m_1^{k+1} = \underset{m_1}{\text{argmin}}{\left\| \left(
\begin{array}{c} m_1^T \\ \left( m_2^k \right)^T \end{array}
\right) Y + R^{k+1} - X + \frac{H^k}{\tau_k} \right\|_F^2} \\
\nonumber  { + \left( m_1^Tm_2^k+\frac{\zeta^k}{\tau_k}\right)^2 }
\end{array}
\label{eq:updatem1}
\end{eqnarray}
This has a closed form solution.

\subsubsection{Update $m_2$}
We discard the terms in $\mathcal{L}_{3}$ which are independent of
$m_2$ and update $m_2$ by:
\begin{eqnarray}
\begin{array}{c}
m_2^{k+1} = \underset{m_2}{\text{argmin}}{\left\| \left(
\begin{array}{c} \left(m_1^{k+1}\right)^T \\  m_2^T \end{array}
\right) Y + R^{k+1} - X + \frac{H^k}{\tau_k} \right\|_F^2} \\
\nonumber { + \left( \left(
m_1^{k+1}\right)^Tm_2+\frac{\zeta^k}{\tau_k}\right)^2}
\end{array}
\label{eq:updatem2}
\end{eqnarray}
This has a closed form solution.

\end{document}